\documentclass[letterpaper, 10 pt, conference]{cls/ieeeconf}
\IEEEoverridecommandlockouts
\let\labelindent\relax

\usepackage[normalem]{ulem}	                        % underlining!
\usepackage[table,usenames,dvipsnames]{xcolor}      % color
\usepackage{enumitem}                               % [inline]
\usepackage{extarrows}                              % http://ctan.org/pkg
\usepackage[noadjust]{cite}                         % Citation
\usepackage{amsmath,amssymb,amsfonts,amsthm,dsfont} % math
\usepackage{algorithm,algorithmicx,listings}        % algorithms
\usepackage[noend]{algpseudocode}			        % necessary for algorithmicx
\usepackage{graphicx,tabularx,subcaption}
\usepackage[export]{adjustbox}% http://ctan.org/pkg/adjustbox
\usepackage{multirow,multicol,array,rotating,diagbox}

\usepackage[font={small}]{caption}
\usepackage[titletoc,title]{appendix}

\usepackage[breaklinks=true, colorlinks, bookmarks=true, citecolor=Black, urlcolor=Violet,linkcolor=Black]{hyperref}

\DeclareMathOperator{\tr}{\mathbf{tr}}

\newcommand{\scaleMathLine}[2][1]{\resizebox{#1\linewidth}{!}{$\displaystyle{#2}$}}
\newcommand{\prl}[1]{\left(#1\right)}
\newcommand{\brl}[1]{\left[#1\right]}

\newcommand{\myhat}[1]{\lfloor#1\times\rfloor}

\newtheorem{proposition}{Proposition}

\theoremstyle{definition}

\newtheorem*{assumption*}{Assumption}
\newtheorem*{problem*}{Problem}

\theoremstyle{remark}

\newtheorem*{solution*}{Solution}

\def\thetitle{Localization and Mapping using Instance-specific Mesh Models}
\def\theauthor{Qiaojun Feng, Yue Meng, Mo Shan, Nikolay Atanasov}
\def\thekeywords{deformable mesh, Semantic Localization and Mapping, multi-view geometry}

\hypersetup{
  pdfauthor={\theauthor},%
  pdftitle={\thetitle},%
  pdfsubject={\thetitle},%
  pdfkeywords={\thekeywords}
}

\IEEEoverridecommandlockouts

\title{\LARGE \bf \thetitle}

\author{Qiaojun Feng $\quad$ Yue Meng $\quad$ Mo Shan $\quad$ Nikolay Atanasov% <-this % stops a space
\thanks{We gratefully acknowledge support from ARL DCIST CRA W911NF-17-2-0181 and ONR N00014-18-1-2828.}%
\thanks{The authors are with Department of Electrical and Computer Engineering, University of California, San Diego, La Jolla, CA 92093, USA {\tt\small \{qjfeng,yum107,moshan,natanasov\}@ucsd.edu}}% <-this % stops a space
}
\begin{document}
\maketitle

\begin{abstract}
This paper focuses on building semantic maps, containing object poses and shapes, using a monocular camera. This is an important problem because robots need rich understanding of geometry and context if they are to shape the future of transportation, construction, and agriculture. Our contribution is an instance-specific mesh model of object shape that can be optimized online based on semantic information extracted from camera images. Multi-view constraints on the object shape are obtained by detecting objects and extracting category-specific keypoints and segmentation masks. We show that the errors between projections of the mesh model and the observed keypoints and masks can be differentiated in order to obtain accurate instance-specific object shapes. We evaluate the performance of the proposed approach in simulation and on the KITTI dataset by building maps of car poses and shapes.
\end{abstract}

% In robotics area Simultaneous Localization and Mapping (SLAM) is a fundamental task that has been well studied. Other than feature point methods that focus more on localization task, recent approaches emphasis on extraction and representation of geometric and semantic information from the environment. In this work, we extract the semantic keypoints and segmentation masks of specific objects and reconstruct the 3D mesh model. The advancement of 3D machine learning approaches can overcome the limits of dependence on existing CAD model, but learn instance-specific 3D mesh reconstruction. After prediction step we can perform further optimization using multi-view observation constraints. Finally we show that such representation and corresponding observation can be  seamlessly integrated into SLAM framework.

\section{Introduction}
\label{sec:introduction}

Autonomous robots bring compelling promises of revolutionizing many aspects of our lives, including transportation, agriculture, mining, construction, security, and environmental monitoring. Transitioning robotic systems from highly structured factory environments to partially known, dynamically changing operational conditions, however, requires perceptual capabilities and contextual reasoning that rival those of biological systems. The foundations of artificial perception lie in the twin technologies of inferring geometry (e.g., occupancy mapping) and semantic content (e.g., scene and object recognition). Visual-inertial odometry (VIO)~\cite{msckf,rovio,vinsmono} and Simultaneous Localization And Mapping (SLAM)~\cite{cadena2016past} are approaches capable of tracking the pose of a robotic system while simultaneously reconstructing a sparse~\cite{lsd-slam,orbslam2} or dense~\cite{Newcombe_DenseVisualSLAM_Phd14,elastic_fusion} geometric representation of the environment. Current VIO and SLAM techniques achieve impressive performance, yet most rely on low-level geometric features such as points~\cite{lowe2004distinctive,rublee2011orb} and planes~\cite{Kaess15icra,quadric_slam} and are unable to extract semantic content. Computer vision techniques based on deep learning recently emerge as a potentially revolutionary way for context comprehension. A major research challenge today is to exploit information provided by deep learning, such as category-specific object keypoints~\cite{newell2016stacked,pavlakos20176}, semantic edges~\cite{casenet}, and segmentation masks~\cite{he2018mask}, in VIO and SLAM algorithms to build rich models of the shape, structure, and function of objects.

This paper addresses camera localization and object-level mapping, incorporating object categories, poses, and shapes. Our main \textbf{contribution} is the development of an instance-specific object shape model based on a triangular mesh and differentiable functions that measure the discrepancy in the image plane between projections of the model and detected semantic information. We utilize semantic keypoints~\cite{newell2016stacked,pavlakos20176,zhou2018starmap} and segmentation masks~\cite{he2018mask} trained on open-source datasets~\cite{xiang_wacv14} as observations for optimizing the error functions. Initialized from a pre-defined mean category-level model, the optimization steps are inspired by the recently proposed differentiable mesh renderer~\cite{kato2018neural}, which allows back-propagation of mask errors measured on a rendered image to update the mesh vertices. We generalize this approach to full $SE(3)$ camera and object pose representations and allow multi-view observation constraints and multi-object reconstruction. The pixel-level information from the segmentation masks is used to incrementally refine the instance-specific object models, which are significantly more accurate than generic category-level ones.

\begin{figure}[t]
  \centering
  \includegraphics[width=\linewidth,trim=0mm 0mm 00mm 00mm, clip]{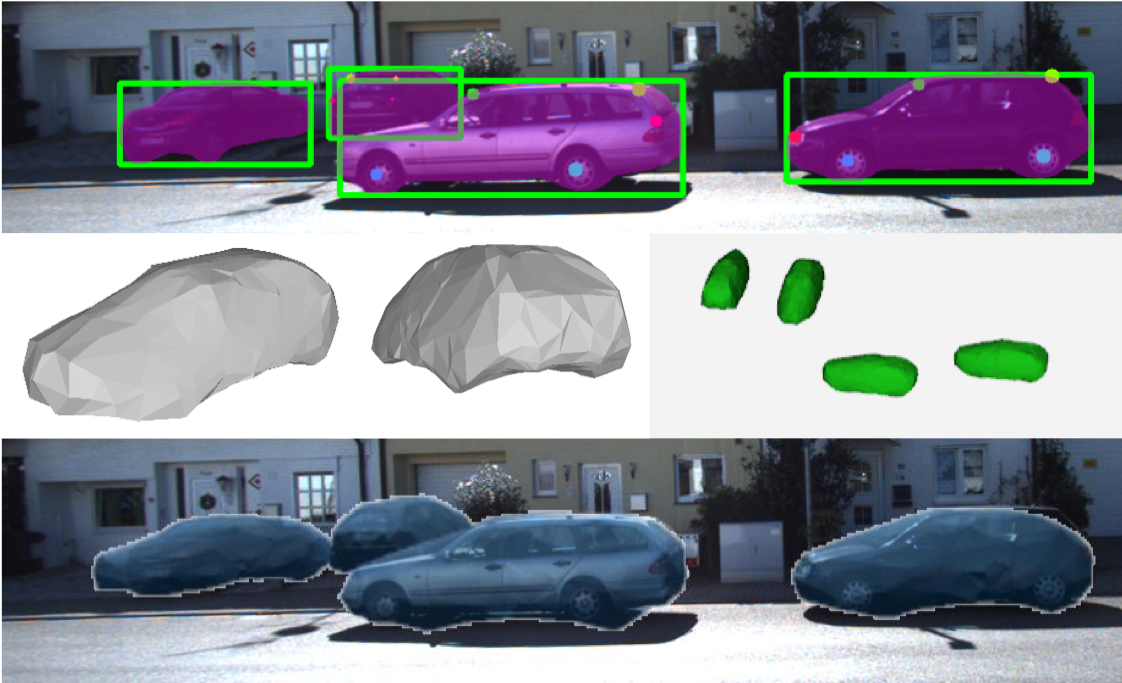}
  \caption{Our objective is to build detailed environment maps incorporating object poses and shapes. The figure from KITTI~\cite{geiger2013vision} in the top row shows the kind of information that our method relies on: bounding boxes (green), segmentation masks (magenta) and semantic keypoints (multiple colors). The middle row includes the reconstructed mesh models and 3D configuration. The last row shows the projection result.}
  \label{fig:obs}
\end{figure}

\section{Related Work}
\label{sec:related_work}

The problem of incorporating semantic information in SLAM has gained increasing attention in recent years~\cite{cadena2016past,Future_SLAM}. In an early approach \cite{civera2011towards}, objects are inserted in the map based on matching of feature descriptors to the models in a database, constructed offline using structure from motion. The camera trajectory provides multi-view information for object pose estimation but the object detections are not used as constraints to optimize the camera trajectory. Recent works often consider the optimization of object and camera poses simultaneously. SLAM++~\cite{salas2013slam++} optimizes the camera and object poses jointly using a factor graph and moreover reconstructs dense surface mesh models of pre-defined object categories. A limitation of this work is that the estimated object shapes are pre-defined and rigid instead of being optimized to match the specific instances detected online.

The popularity of joint optimization of camera and object poses keeps increasing with the advent of robust 2-D object detectors based on structured models~\cite{dpm} and deep neural networks~\cite{RCNN,resnet}. The stacked hourglass model~\cite{newell2016stacked} is used by several works~\cite{pavlakos20176,Atanasov_SemanticSLAM_IJCAI18} to extract mid-level object parts and, in turn, perform factor graph inference to recover the global positions and orientations of objects detected from a monocular camera. In~\cite{fei2018visual}, a deep network object detector is used to generate object hypotheses, which are subsequently validated using geometric and semantic cues and optimized via nonlinear filtering. Some of these approaches~\cite{semslam,Atanasov_SemanticSLAM_IJCAI18,fei2018visual} use inertial measurements and probabilistic data association among detections and objects as additional constraints in the optimization. While most approaches focus on static objects, \cite{li2018stereo} uses a stereo camera to track ego-motion and dynamic 3-D objects in urban driving scenarios. The authors use bundle adjustment to fuse 3-D object measurements obtained from detection and viewpoint classification.

%\subsection{Object Representation and Reconstruction}
Various 3D object representations, including pointclouds~\cite{fan2017point,tulsiani2017learning}, voxels~\cite{yan2016perspective,tulsiani2017multi}, meshes~\cite{kato2018neural,kanazawa2018learning}, and quadrics~\cite{nicholson2018quadricslam,quadric_slam}, have been proposed in the literature. We are particularly interested in object models that can be constrained from multi-view observations and can adapt their pose and shape to specific object instances observed online. Tulsiani, Kar, et al.~\cite{tulsiani2017learning} learn a deformable pointcloud model with mean shape and deformation bases to fit object silhouettes at test time. The perspective transformer nets~\cite{yan2016perspective} use perspective projection to synthesize observations from different views and can be utilized for multi-view shape optimization. Introducing object models into the online inference process of SLAM requires compact representations that can be optimized and stored efficiently. QuadricSLAM~\cite{nicholson2018quadricslam,quadric_slam} parameterizes objects using dual ellipsoids, which can be extracted directly from bounding box detections and optimized using few parameters. A triangular mesh model of object shapes is proposed by~\cite{kanazawa2018learning} and is optimized from a single image using object keypoints and silhouettes. The optimization proceses uses approximate gradients of a mesh rasterization function obtained via the neural mesh renderer~\cite{kato2018neural}. In this work, we generalize the deformable mesh model to $SE(3)$ camera and object poses and allow multi-view constraints and multi-object reconstruction.

\section{Problem Formulation}
\label{sec:problem_formulation}

We consider the problem of  detecting, localizing, and estimating the shape of object instances present in the scene, and estimating the pose of a camera over time. The states we are interested in estimating are the camera poses $\mathcal{C}\triangleq\{c_t\}_{t=1}^T$ with $c_t \in SE(3)$ and the object shapes and poses $\mathcal{O}\triangleq\{o_n\}_{n=1}^N$. More precisely, a camera pose $c_t := (R_{c_t},p_{c_t})$ is determined by its position $p_{c_t} \in \mathbb{R}^3$ and orientation $R_{c_t} \in SO(3)$, while an object state $o_n = (\mu_n,R_{o_n},p_{o_n})$ consists of a pose $R_{o_n} \in SO(3)$, $p_{o_n} \in \mathbb{R}^3$ and shape $\mu_n$, specified as a 3-D triangular mesh $\mu_n = (V_n,F_n)$ in the object canonical frame with vertices $V_n\in\mathbb{R}^{3\times|V_n|}$ and faces $F_n \in \mathbb{R}^{3 \times |F_n|}$. Each row of $F_n$ contains the indices of 3 vertices that form a triangular face. A subset of the mesh vertices are designated as keypoints -- distinguishing locations on an object's surface (e.g., car door, windshield, or tires) that may be detected using a camera sensor. We define a keypoint association matrix $A_n\in\mathbb{R}^{|V_n|\times|K_n|}$ that generates $|K_n|$ keypoints $K_n=V_n A_n$ from all mesh vertices.

Suppose that a sequence $\mathcal{I}\triangleq\{i_t\}^T_{t=1}$ of $T$ images $i_t \in \mathbb{R}^{W \times H}$, collected from the corresponding camera poses $\{c_t\}_{t=1}^T$, are available for the estimation task. From each image $i_t$, we extract a set of object observations $\mathcal{Z}_t\triangleq\{z_{lt} = (\xi_{lt},s_{lt},y_{lt})\}_{l=1}^{L_t}$, consisting of a detected object's category $\xi_{lt} \in \Xi$, a segmentation masks $s_{lt} \in \{0,1\}^{W \times H}$ and the pixel coordiantes of detected keypoints $y_{lt} \in \mathbb{R}^{2\times|K_{lt}|}$. We suppose that $\Xi$ is a finite set of pre-defined detectable object categories and that the data association $n = \pi_t(l)$ of observations to object instances is known (we describe an object tracking approach in Sec.~\ref{sec:approach_perception} but global data association can also be performed~\cite{jcbb,kaess_hungarian,semslam}). See Fig.~\ref{fig:obs} for example object observations.

% \begin{figure}[t]
%   \centering
%   \includegraphics[width=0.54\linewidth]{fig/graph.png}
%   \includegraphics[width=0.38\linewidth]{fig/factor_graph.png}
%   \caption{The factor graph of two camera poses $c_t,c_{t+1}$ and two objects $o_1,o_2$, constrained by observation $z_{t,1},z_{t,2},z_{t+1,2}$ and IMU measurement}
%   \label{fig:factor_graph}
% \end{figure}

For a given estimate of the camera pose $\hat{c}_t$ and the object state $\hat{o}_n$, we can predict expected semantic mask $\hat{s}_{lt}$ and semantic keypoint observations $\hat{y}_{lt}$ using a perspective projection model:
\begin{equation}
\label{eq:projection_model}
\begin{aligned}
\hat{s}_{lt} &= \mathcal{R}_{\text{mask}}(\hat{c}_t,\hat{o}_n) \\
\hat{y}_{lt} &= \mathcal{R}_{\text{kps}}(\hat{c}_t,\hat{o}_n,A_n)
\end{aligned}
\end{equation}
where the mask and keypoint projection functions $\mathcal{R}_{\text{mask}}$, $\mathcal{R}_{\text{kps}}$ will be defined precisely in Sec.~\ref{sec:technical_approach}. The camera and object estimates can be optimized by reducing the error between the predicted $\hat{\mathcal{Z}}_{1:T}$ and the actual $\mathcal{Z}_{1:T}$ observations. We define loss functions $\mathcal{L}_{\text{mask}}$, measuring discrepancy between semantic masks, and $\mathcal{L}_{\text{kps}}$, measuring discrepancy among semantic keypoints, as follows:
\begin{equation}
\label{eq:loss_functions}
\begin{aligned}
\mathcal{L}_{\text{mask}}(s,\hat{s}) =& -\frac{\|s \odot \hat{s}\|_1}{\|s + \hat{s} - s \odot \hat{s}\|_1} \\
\mathcal{L}_{\text{kps}}(y,\hat{y}) =& \|y - \hat{y} \cdot vis(\hat{y}) \|_F^2
\end{aligned}
\end{equation}
where $\odot$ is an element-wise product and $vis(\hat{y}_{lt}) \in \{0,1\}^{|K_n|\times|K_{lt}|}$ is a binary selection matrix that discards unobservable object keypoints.
%and implicitly depends on $\hat{c}_t$ and the data association $n = \pi_t(l)$.

\begin{problem*}
Given object observations $\mathcal{Z}_{1:T}$, determine the camera poses $\mathcal{C}$ and object states $\mathcal{O}$ that minimize the mask and keypoint losses:
\begin{equation}
\label{eq:problem}
\begin{aligned}
\min_{\mathcal{C}, \mathcal{O}}\sum_{t=1}^{T} \sum_{l=1}^{L_t} &\left( w_{\text{mask}}\mathcal{L}_{\text{mask}}(s_{lt}, \mathcal{R_{\text{mask}}}(c_t,o_{\pi_t(l)})) \right.\\
&+ \left. w_{\text{kps}}\mathcal{L}_{\text{kps}}\prl{y_{lt},\mathcal{R}_{\text{kps}}(c_t,o_{\pi_t(l)},A_{\pi_t(l)})}\right)
\end{aligned}
\end{equation}
where $w_{\text{mask}}$, $w_{\text{kps}}$ are scalar weight parameters specifying the relative importance of the mask and keypoint loss functions. 
\end{problem*}

\section{Technical Approach}
\label{sec:technical_approach}

We begin by describing how the object observations $\mathcal{Z}_t$ are obtained. Next, we provide a rigorous definition of the perspective projection models in~\eqref{eq:projection_model}, which, in turn, define the loss functions in~\eqref{eq:loss_functions} precisely. Finally, in order to perform the optimization in~\eqref{eq:problem}, we derive the gradients of $\mathcal{L}_{\text{mask}}$ and $\mathcal{L}_{\text{kps}}$ with respect to $c_t$ and $o_t$.

\subsection{Semantic Perception}
\label{sec:approach_perception}

We extract both category-level (object category $\xi_{lt}$ and keypoints $y_{lt}$) and instance-level (segmentation masks $s_{lt}$) semantic information from the camera images. 
For each frame, we first use pre-trained model~\cite{massa2018mrcnn} to get object detection results represented with bounding boxes and instance segmentations inside the boxes. Each object is assigned to one of the class labels in $\Xi$. Then we extract semantic keypoints $y_{lt}$ within the bounding box of each detected object using the pre-trained stacked hourglass model of~\cite{zhou2018starmap}, which is widely used for human-joint/object-keypoint detector.
The $l$-th detection result at time $t$ contains the object category $\xi_{lt} \in \Xi$, keypoints $y_{lt} \in \mathbb{R}^{2 \times |K_{lt}|}$, mask $s_{lt} \in \{0,1\}^{W \times H}$, bounding box $\beta_{lt} \in \mathbb{R}^4$ (2-D location, width, and height) as shown in Fig.~\ref{fig:obs}, object detection confidence $u_{lt} \in \mathbb{R}$ and keypoint detection confidences $q_{lt} \in \mathbb{R}^{|K_{lt}|}$.

We develop an object tracking approach in order to associate the semantic observations obtained over time with the object instance that generated them. We extend the KLT-based ORB-feature-tracking method of~\cite{sun2018robust} to track semantic features $y_{lt}$ by accounting for their individual labels (e.g., car wheel, car door) and share object category $\xi_{lt}$. In detail, let $z_{lt}$ be a semantic observation from a newly detected object at time $t$. The objective is to determine if $z_{lt}$ matches any of the semantic observations $z_{m,t+1} \in \mathcal{Z}_{t+1}$ at time $t+1$ given that both have the same category label, i.e., $\xi_{lt} = \xi_{m,t+1}$. We apply the KLT optical flow algorithm~\cite{lucas1981iterative} to estimate the locations $y_{l,t+1}$ of the semantic features $y_{lt}$ in the image plane at time $t+1$. We use the segmentation mask $s_{m,t+1}$ of the $m$-th semantic observation to determine if $y_{l,t+1}$ are inliers (i.e., if the segmentation mask $s_{m,t+1}$ is $1$ at image location $y_{l,t+1}$) with respect to observation $m$. Let $in(y_{l,t+1},s_{m,t+1}) \in \{0,1\}^{|K_{lt}|}$ return a binary vector indicating whether each keypoint is an inlier or not. We repeat the process in reverse to determine if the backpropagated keypoints $y_{m,t}$ of observation $m$ are inliers with respect to observation $l$. Eventually, we compute a matching score based on the inliers and their detection confidences:
\begin{equation}
M_{lm} = \sum_{k=1}^{K_{lt}} in(y_{l,t+1}^{(k)},s_{m,t+1}^{(k)}) \cdot in(y_{m,t}^{(k)},s_{l,t}^{(k)}) \cdot q_{lt}^{(k)}
\end{equation}
where $q_{lt}^{(k)}$ is the $k$-th element of $q_{lt}$. Finally, we match observation $l$ to the observation at time $t+1$ that maximizes the score, i.e., $m^* = \arg\max_k M_{lm}$. If the object bounding boxes $\beta_{lt}$ and $\beta_{m^*,t+1}$ have compatible width and height, we declare that object $l$ has been successfully tracked to time $t+1$. Otherwise, we declare that object track $l$ has been lost.

\subsection{Mesh Renderer as an Observation Model}
\label{sec:approach_renderer}

Next, we develop the observation models $\mathcal{R}_{\text{mask}}$ and $\mathcal{R}_{\text{kps}}$ that specify how a semantic observations $z = (\xi,s,y)$ is generated by a camera with pose $(R_c,p_c) \in SE(3)$ observing an object of class $\xi \in \Xi$ with pose $(R_o,p_o) \in SE(3)$ and mesh shape $\mu = (V,F)$ with keypoint association matrix $A$. Let $K$ be the intrinsic matrix of the camera, defined as:
\begin{equation}
\label{eq:intrinsic_matrix}
K=\begin{bmatrix}f s_u&f s_n&c_u\\0&f s_v&c_v\end{bmatrix} \in \mathbb{R}^{2\times 3},
\end{equation}
where $f$ is the focal length in meters, $(s_u,s_v)$ is the pixels per meter resolution of the image array, $(c_u,c_v)$ is the image center in pixels and $s_n$ is a rectangular pixel scaling factor. 
Let $x := VAe_k \in \mathbb{R}^3$ be the coordinates of the $k$-th object keypoint in the object frame, where $e_k$ is a standard basis vector. The projection of $x$ onto the image frame can be determined by first projecting it from the object frame to the camera frame using $(R_o,p_o)$ and $(R_c,p_c)$ and then the perspective projection $\pi: \mathbb{R}^3 \rightarrow \mathbb{R}^3$, and the linear transformation $K$. In detail, this sequence of transformations leads to the pixel coordinates of $x$ as follows:
\begin{equation}
y^{(k)} = K\pi(R_c^T(R_ox + p_o-p_c)) \in \mathbb{R}^2
\end{equation}
where the standard perspective projection function is:
\begin{equation}
\pi(x) = \begin{bmatrix}x_1 / x_3&x_2 / x_3&x_3 / x_3\end{bmatrix}^T
\end{equation}
Applying the same transformation to all object keypoints $VA$ simultaneously leads to the keypoint projection model:
\begin{equation}
\label{eq:kps_proj}
\begin{aligned}
\mathcal{R}_{\text{cam}}(c,o,A) &:= R_c^T(R_oVA + (p_o-p_c) \mathbf{1}^T)\\
\mathcal{R}_{\text{kps}}(c,o,A) &:= K\pi\mathcal{R}_{\text{cam}}(c,o,A)
\end{aligned}
\end{equation}
where $\mathbf{1}$ is a vector whose elements are all equal to $1$.

To define $\mathcal{R}_{\text{mask}}$, we need an extra rasterization step, which projects the object faces $F$ to the image frame. A rasterization function, $Raster(\cdot)$, can be defined using the standard method in~\cite{marschner2015fundamentals}, which assumes that if multiple faces are present, only the frontmost one is drawn at each pixel. Kato et al.~\cite{kato2018neural} also show how to obtain an approximate gradient for the rasterization function. Relying on \cite{marschner2015fundamentals} and \cite{kato2018neural} for $Raster(\cdot)$, we can define the mask projection model:
\begin{equation}
\label{eq:mask_proj}
\mathcal{R}_{\text{mask}}(c,o) := Raster\prl{ \mathcal{R}_{\text{cam}}(c,o,I), F }
\end{equation}

Now that the projection models~\eqref{eq:projection_model} and hence the loss functions~\eqref{eq:loss_functions} have been well defined, the final step needed to perform the optimization in~\eqref{eq:problem} is to derive their gradients. We assume that the connectivity $F$ of the object mesh is fixed and the mesh is deformed only by changing the locations of the vertices $V$. We use the results of~\cite{kato2018neural} for the gradient $\nabla_V Raster(V,F)$. Since $\mathcal{R}_{\text{mask}}$ is a function of $\mathcal{R}_{\text{cam}}$ according to~\eqref{eq:mask_proj}, we only need to derive the following:
\begin{equation}
\begin{aligned}
&\nabla_{\hat{s}} \mathcal{L}_{\text{mask}}(s,\hat{s}),\; &&\nabla_{\hat{y}} \mathcal{L}_{\text{kps}}(y,\hat{y}),\\
&\nabla_c \mathcal{R}_{\text{kps}}(c,o,A),\;&&\nabla_o \mathcal{R}_{\text{kps}}(c,o,A).
\end{aligned}
\end{equation}
Our results are summarized in the following propositions.

\begin{proposition}
\label{prop:loss_grad}
The gradients of the loss functions $\mathcal{L}_{\text{mask}}(s,\hat{s})$ and $\mathcal{L}_{\text{kps}}(y,\hat{y})$ in~\eqref{eq:loss_functions} with respect to the estimated mask $\hat{s} \in \{0,1\}^{W \times H}$ and keypoint pixel coordinates $\hat{y} \in \mathbb{R}^{2 \times K}$ are:
\begin{align}
\nabla_{\hat{y}} \mathcal{L}_{\text{kps}}(y,\hat{y}) & = 2\,(\hat{y} \cdot vis(\hat{y}) - y)\,vis(\hat{y})^T\\
\nabla_{\hat{s}} \mathcal{L}_{\text{mask}}(s,\hat{s}) & = -\frac{1}{U(s,\hat{s})}\cdot s + \frac{I(s,\hat{s})}{U^2(s,\hat{s})}\cdot(\mathbf{1}\mathbf{1}^T-s)\notag
\end{align}
where $I(s,\hat{s}) := \|s \odot \hat{s}\|_1$ and $U(s,\hat{s}) := \|s + \hat{s} - s \odot \hat{s}\|_1$.
\end{proposition}
\iffalse
\begin{proof}
We assume that $G := vis(\hat{y})$ remains constant for an infinitesimal change in $\hat{y}$. Then, 
\begin{equation}
\nabla_{\hat{y}} \mathcal{L}_{\text{kps}}(y,\hat{y}) = 2\nabla_{\hat{y}} \tr\prl{(y-\hat{y}G)^T(y-\hat{y}G)}
\end{equation}
Rewriting $\mathcal{L}_{\text{mask}}(s,\hat{s})$ as $- I(s,\hat{s}) / U(s,\hat{s})$ leads to:
\begin{align*}
\nabla_{\hat{s}}\brl{\frac{I(s,\hat{s})}{U(s,\hat{s})}} &= \frac{\nabla_{\hat{s}}I(s,\hat{s})U(s,\hat{s}) - I(s,\hat{s})\nabla_{\hat{s}}U(s,\hat{s})}{U^2(s,\hat{s})} \\
\nabla_{\hat{s}}I(s,\hat{s}) &= s, \; \nabla_{\hat{s}}U(s,\hat{s}) = \mathbf{1}\mathbf{1}^T - s\qedhere
\end{align*}
\end{proof}
\fi
\begin{proposition}
\label{prop:obs_grad}
Let $y^{(i)} = \mathcal{R}_{\text{kps}}(c,o,I) \in \mathbb{R}^2$ be the pixel coordinates of the $i$-th vertex $v_i := Ve_i$ of an object with pose $(R_o,p_o) \in SE(3)$ obtained by projecting $v_i$ onto the image plane of a camera with pose $(R_c,p_c) \in SE(3)$, calibration matrix $K\in \mathbb{R}^{2\times 3}$. Let $\theta_c,\theta_o \in \mathbb{R}^3$ be the axis-angle representations of $R_c$ and $R_o$, respectively, so that $R_c = \exp(\myhat{\theta_c})$ and $R_o = \exp(\myhat{\theta_o})$ and $\myhat{\cdot}$ is the hat map. Then, the derivative of $y^{(i)}$ with respect to $\alpha \in \{\theta_c, p_c, \theta_o, p_o, v_i\}$ is:
\begin{equation}
\frac{\partial y^{(i)}}{\partial \alpha} = K \frac{\partial \pi}{\partial x}(\gamma) \frac{\partial \gamma}{\partial \alpha}
\end{equation}
where:
\begin{align}
\frac{\partial \pi}{\partial x}(x) &= \frac{1}{x_3}\begin{bmatrix}1&0&-x_1/x_3\\0&1&-x_2/x_3\\0&0&0\end{bmatrix}\notag\\
\gamma &= R_c^T(R_ov_i + p_o-p_c)\\
\frac{\partial \gamma}{\partial p_c} &= -R_c^T \quad \frac{\partial \gamma}{\partial p_o} = R_c^T \quad \frac{\partial \gamma}{\partial v_i} = R_c^TR_o\notag\\
\frac{\partial \gamma}{\partial \theta_c} &=R_c^T \myhat{(R_ov_i + p_o-p_c)}J_{rSO(3)}(-\theta_c)\notag\\
\frac{\partial \gamma}{\partial \theta_o}& = -R_c^T R_o\myhat{{v}_i} J_{rSO(3)}(\theta_o)\notag
\end{align}
and $J_{rSO(3)}(\theta)$ is the right Jacobian of $SO(3)$, which is necessary because the gradient needs to be projected from the tangent space to the $SO(3)$ manifold~\cite[Ch. 10.6]{right_jacobian_so3}, and can be computed in closed form:
\begin{equation}
J_{rSO(3)}(\theta) = I_3 - \frac{1-\cos\|\theta\|}{\|\theta\|^2}\myhat{\theta} + \frac{\|\theta\|-\sin\|\theta\|}{\|\theta\|^3}\myhat{\theta}^2.
\end{equation}
\end{proposition}

\begin{proof}
By definition~\eqref{eq:kps_proj}, $y^{(i)} = K\pi(\gamma)$ so most steps follow by the chain rule. We only discuss the relationship between the axis-angle vectors $\theta_c$, $\theta_o$ and the orientations $R_c$, $R_o$. Any rotation matrix $R\in SO(3)$ can be associated with a vector $\theta\in\mathbb{R}^3$ specifying it as a rotation about a fixed axis $\frac{\theta}{\|\theta\|_2}$ through an angle $\|\theta\|_2$. The axis-angle representation $\theta$ is related to $R$ through the exponential and logarithm maps:
\begin{align*}
&\scaleMathLine{R = \exp(\myhat{\theta}) = I + \left(\frac{\sin\|\theta\|}{\|\theta\|}\right)\myhat{\theta} + \left(\frac{1-\cos\|\theta\|}{\|\theta\|^2}\right)\myhat{\theta}^2} \\
&\myhat{\theta} = \log(R) = \frac{\|\theta\|}{2\sin\|\theta\|}(R-R^T)
\end{align*}
%where the hat map is defined as
%\begin{equation}
%\myhat{\theta} = \begin{bmatrix}0&-\theta_3&\theta_2\\\theta_3&0&-\theta_1\\-\theta_2&\theta_1&0\end{bmatrix} \in so(3)
%\end{equation}
See~\cite{invitation_3D_vision} and \cite{right_jacobian_so3} for details. Consider the derivative of $\gamma$ with respect to $\theta_o$. The right Jacobian of $SO(3)$ satisfies the following for small $\delta \theta$:
\[
\exp(\myhat{(\theta + \delta\theta)}) \approx \exp(\myhat{\theta}) \exp(\myhat{(J_{rSO(3)}(\theta)\delta\theta)})
\]
Using this and $R_o = \exp(\myhat{{\theta}_o})$, we can compute:
\begin{equation}
\begin{aligned}
\frac{\partial \gamma}{\partial \theta_o} &= R_c^T\frac{\partial \exp(\myhat{\theta_o})v_i}{\partial \theta_o}\\
&= R_c^TR_o \frac{\partial }{\partial \delta \theta_o} \myhat{(J_{rSO(3)}(\theta_o)\delta_o\theta)}v_i\\
&= - R_c^TR_o \myhat{v}_i J_{rSO(3)}(\theta_o) \frac{\partial \delta \theta_o}{\partial \delta \theta_o}
\end{aligned}
\end{equation}
The derivative of $\gamma$ with respect to $\theta_c$ can be obtained using the same approach.
\end{proof}
In conclusion, we derived explicit definitions for the observation models $\mathcal{R}_{\text{kps}}$, $\mathcal{R}_{\text{mask}}$, the loss functions $\mathcal{L}_{\text{mask}}$, $\mathcal{L}_{\text{kps}}$, and their gradients directly taking the $SO(3)$ constraints into account via the axis-angle parameterization. As a result, we can treat~\eqref{eq:problem} as an unconstrained optimization problem and solve it using gradient descent. The explicit gradient equations in Prop.~\ref{prop:obs_grad} allow solving an object mapping-only problem by optimizing with respect to $\mathcal{O}$, a camera localization-only problem by optimizing with respect to $\mathcal{C}$, or a simultaneous localization and mapping problem.

%%%%%%%%%%%%%%%%%%%%%%%%%%%%%%%%%%%%%%%%%%%%%%%%%%%%%%%%%%
\subsection{Optimization Initialization}
\label{sec:approach_details}
We implemented the localization and mapping tasks separately. In the localization task, we initialize the camera pose using inertial odometry obtained from integration of IMU measurements~\cite{msckf}. The camera pose is optimized sequentially between every two images via~\eqref{eq:problem}, leading to an object-level visual-inertial odometry algorithm.

To initialize the object model in the mapping task, we collect high-quality keypoints (according to $q_{lt}$ defined in Sec.~\ref{sec:approach_perception}) from multiple frames until an object track is lost. The 3-D positions of these keypoints are estimated by optimizing $\mathcal{L}_{\text{kps}}$ only using the Levenberg-Marquardt algorithm. Using a predefined category-level mesh model(mean model) with known keypoints, we apply the Kabsch algorithm \cite{kabsch1976solution} to initialize the object pose (i.e., the transformation from the detected 3-D keypoints to the category-level model keypoints). 
%In case of low-quality keypoint detections, we also train a category-specific mesh reconstruction model~\cite{kanazawa2018learning} to predict an object's mesh shape and pose from a single image. By comparing the silhouette of the predicted object mesh with the observed segmentation masks, we can choose a valid prediction for initialization. More importantly, this model can capture some instance-specific features to provides better shape initialization. 
After initialization, we take two steps to optimize the object states. First, we fix the mesh vertices and optimize the pose based on the combined loss function in~\eqref{eq:problem}. Next, we fix the object pose, and optimize the mesh vertices using only the mask loss because the keypoint loss affects only few vertices. To improve the deformation optimization and obtain a smooth mesh model, we add regularization using the mean mesh curvature. The curvature is computed using a discretization of the continuous Laplace-Beltrami operator~\cite{pinkall1993computing,sorkine2006differential}. Constraints from symmetric object categories can be enforced by directly defining the mesh shape model to be symmetric.

\begin{figure}[t]
  \centering
  \includegraphics[width=0.51\linewidth,trim=10mm 0mm 0mm 10mm, clip]{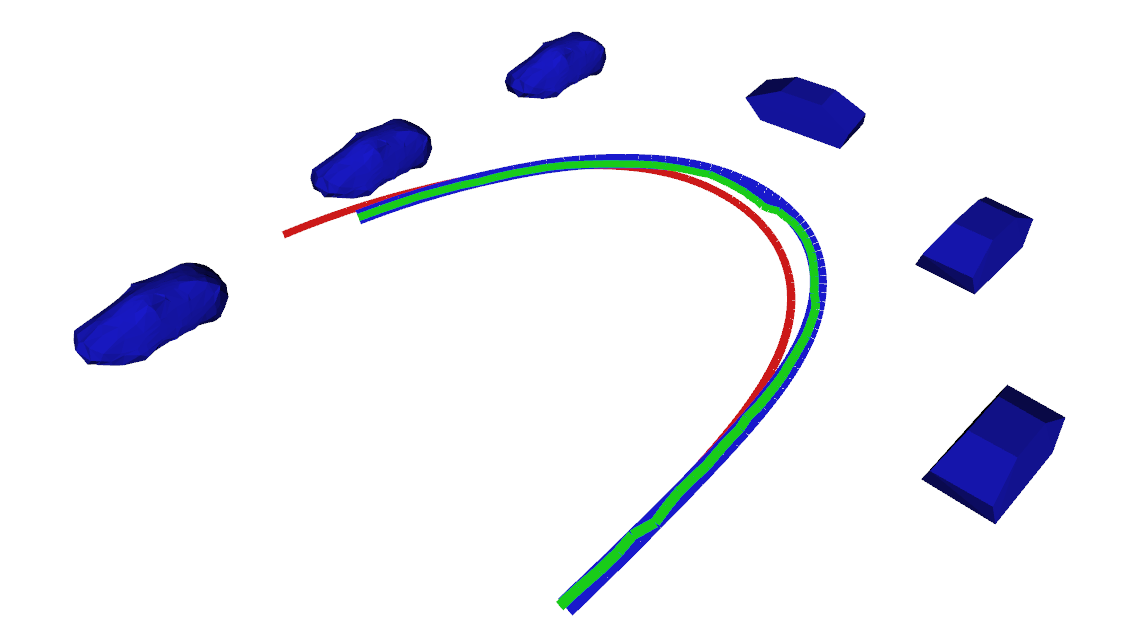}
  \includegraphics[width=0.47\linewidth,trim=0mm 0mm 10mm 10mm, clip]{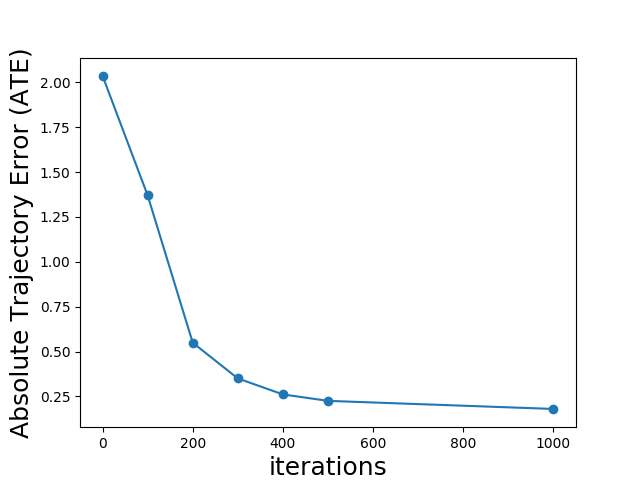}
  \caption{Left: Localization results from a simulated dataset, showing car poses (blue), the ground truth camera trajectory (blue), the inertial odometry used for initialization (red), and the optimized camera trajectory (green). Right: Change of Absolute Trajectory Error versus number of optimization iterations.}
  \label{fig:loc_sim}
\end{figure}

\section{Experiments}
\label{sec:experiments}

\begin{figure}[t]
  \centering
  \includegraphics[width=\linewidth,valign=t]{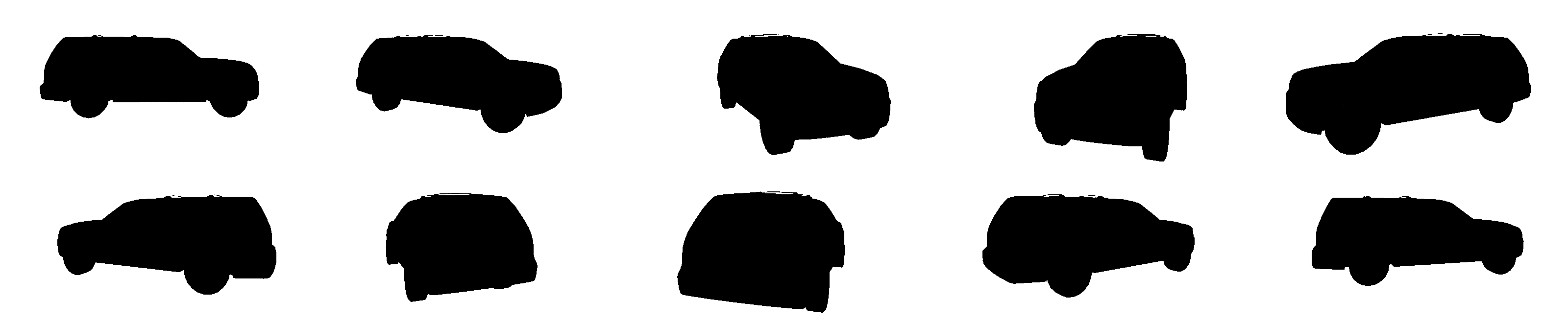}
  \caption{Views used to evaluate the object-level mapping approach in simulation.}
  \label{fig:map_sim_viewpoints}
\end{figure}

\begin{figure}[t]
  \centering
  \includegraphics[width=0.9\linewidth,trim=0mm 0mm 0mm 0mm, clip]{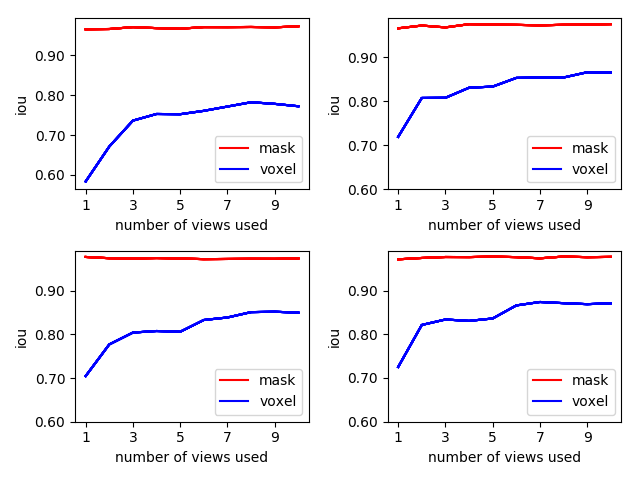}
  \caption{Mask and 3-D voxel intersection over union (IoU) results obtained with different numbers of object views for four different object instances (see Fig.~\ref{fig:map_sim}).}
  % Notice that IoU of mask only take used masks into account.
  \label{fig:map_sim_stats}
\end{figure}

% Adjust box alignment
\begin{figure*}[t]
  \centering
  \includegraphics[width=0.8\linewidth]{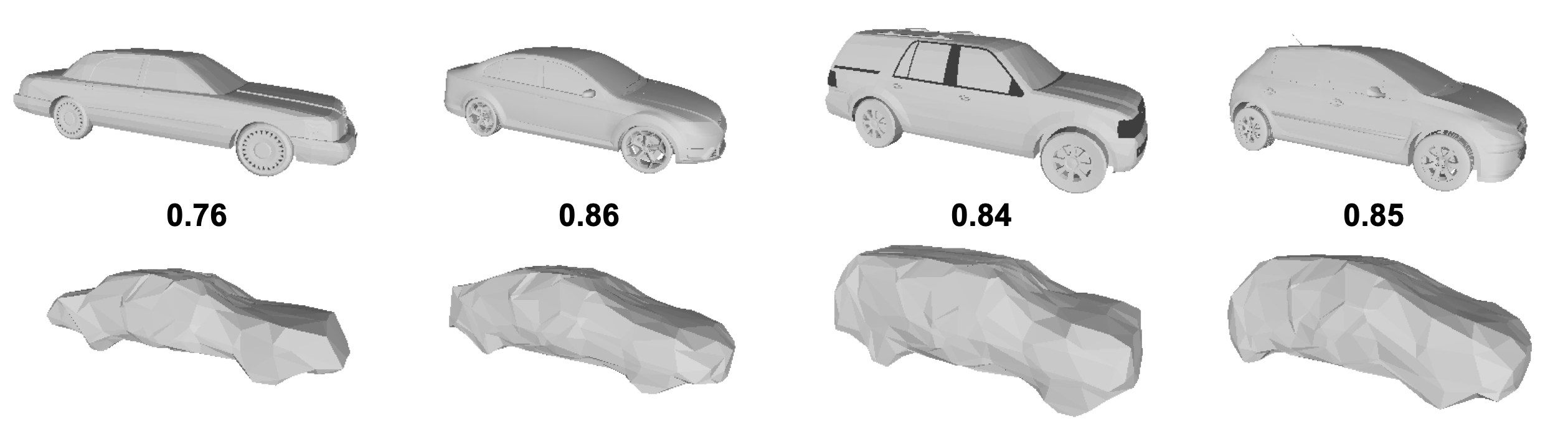}
  \caption{Qualitative comparison between estimated car shapes (bottom row) obtained from a simulation sequence and ground truth object meshes (top row). The numbers indicate 3D IoU.}
  \label{fig:map_sim}
\end{figure*}

% Adjust box alignment
\begin{figure*}[t]
  \centering
  \includegraphics[width=\linewidth,trim=0mm 0mm 0mm 5mm, clip]{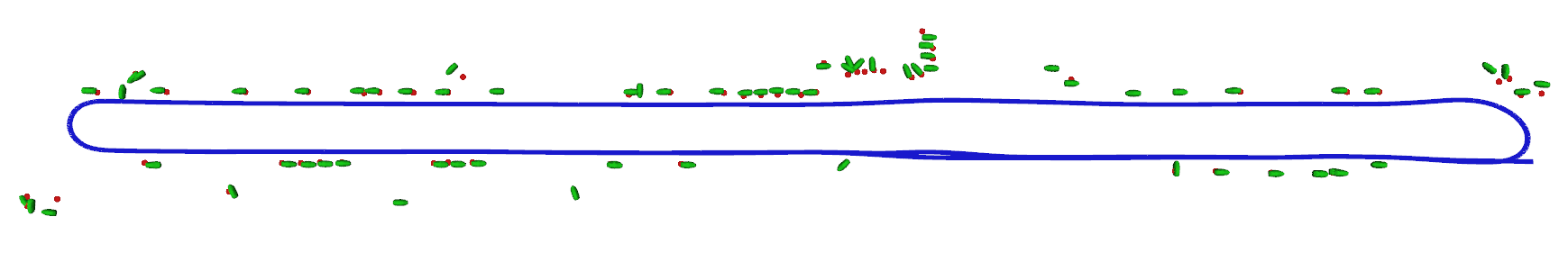}
  \caption{Qualitative results showing the accuracy of the estimated car positions (green) on sequence 06 of the KITTI odometry dataset in comparison with hand-labeled ground truth (red) obtained from~\cite{atanasov2016localization}. The camera trajectory (blue) is shown as well.}
  \label{fig:loc_KITTI}
\end{figure*}

% Adjust box alignment
\begin{figure*}[t]
  \centering
  \includegraphics[width=0.326\linewidth,valign=t,trim=0mm 0mm 0mm 40mm, clip]{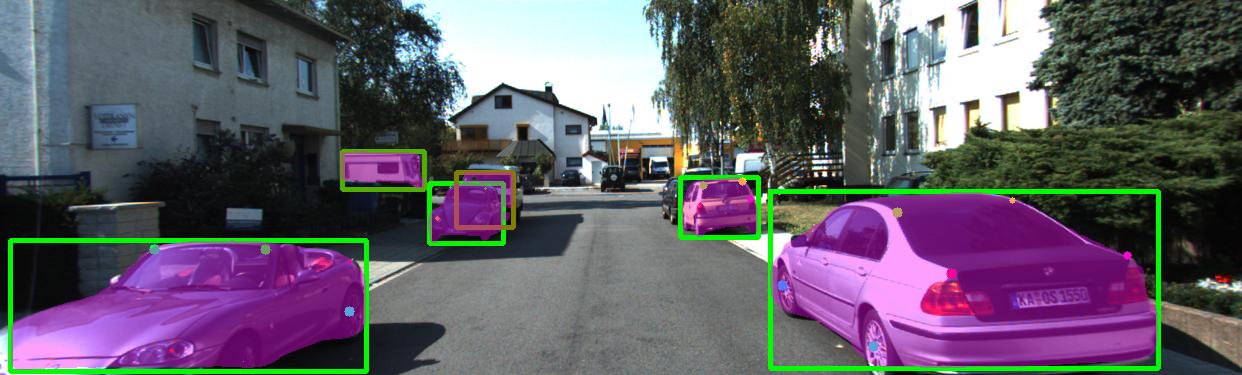}
  \includegraphics[width=0.326\linewidth,valign=t,trim=0mm 0mm 0mm 40mm, clip]{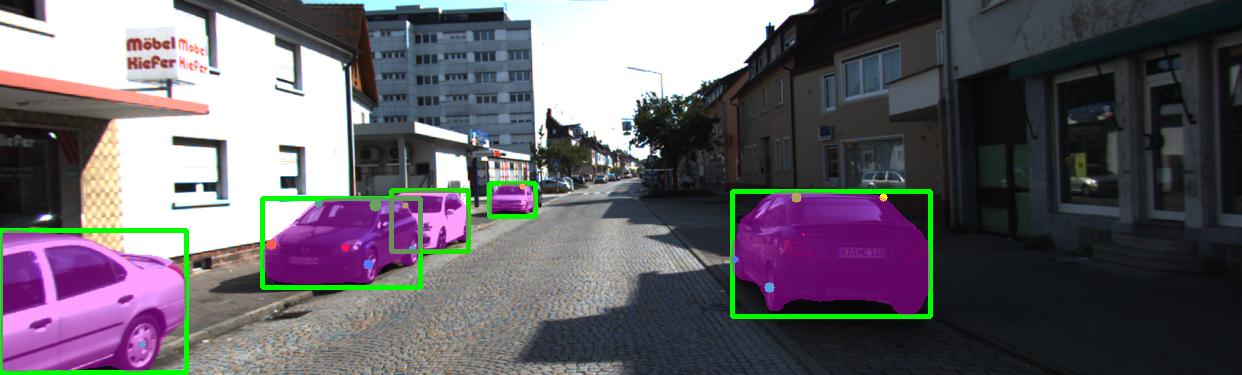}
  \includegraphics[width=0.326\linewidth,valign=t,trim=0mm 0mm 0mm 40mm, clip]{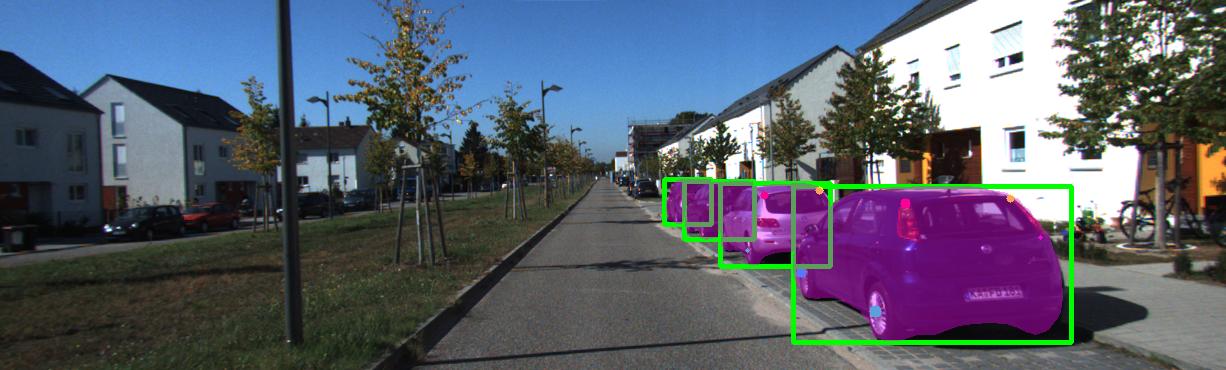}
  \includegraphics[width=0.326\linewidth,valign=t,trim=0mm 0mm 0mm 40mm, clip]{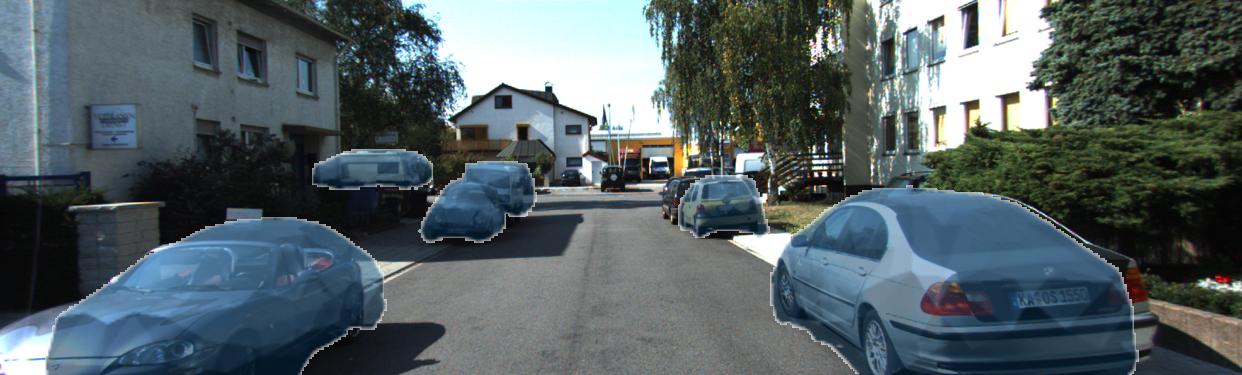}
  \includegraphics[width=0.326\linewidth,valign=t,trim=0mm 0mm 0mm 40mm, clip]{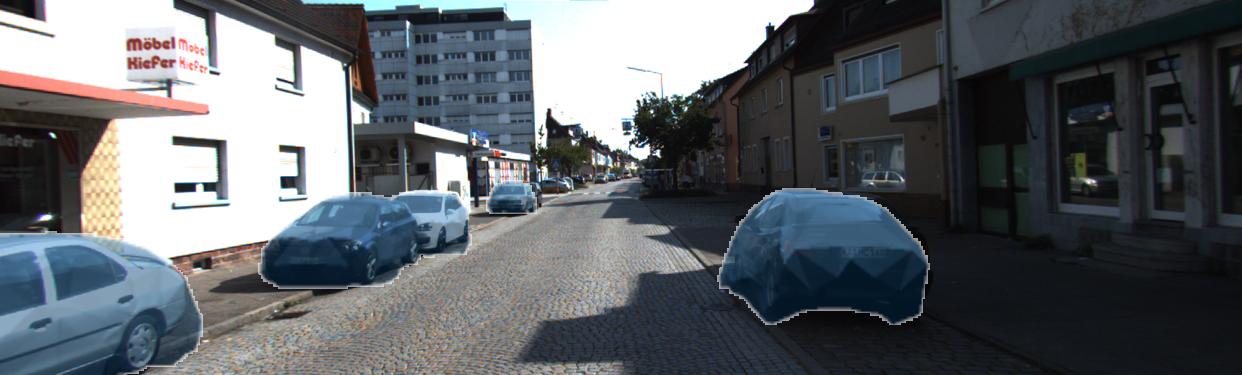}
  \includegraphics[width=0.326\linewidth,valign=t,trim=0mm 0mm 0mm 40mm, clip]{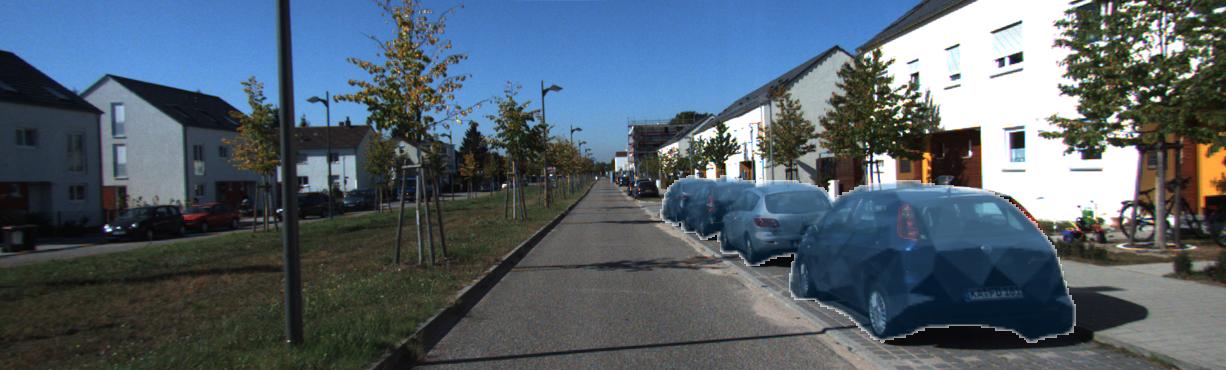}
  \caption{Top: the semantic observations. Bottom: the projection of reconstructed mesh models.}
  \label{fig:map_KITTI}
\end{figure*}

\begin{figure}[t]
  \centering
  \includegraphics[width=0.32\linewidth,trim=0mm 0mm 0mm 0mm, clip]{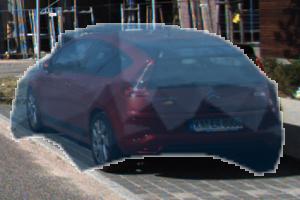}
  \includegraphics[width=0.32\linewidth,trim=0mm 0mm 0mm 0mm, clip]{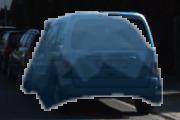}
  \includegraphics[width=0.32\linewidth,trim=0mm 0mm 0mm 0mm, clip]{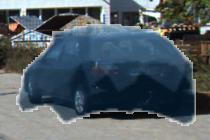}
  \includegraphics[width=0.32\linewidth,trim=0mm 0mm 0mm 0mm, clip]{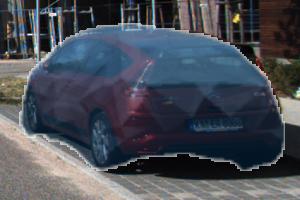}
  \includegraphics[width=0.32\linewidth,trim=0mm 0mm 0mm 0mm, clip]{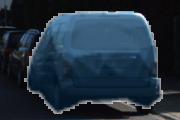}
  \includegraphics[width=0.32\linewidth,trim=0mm 0mm 0mm 0mm, clip]{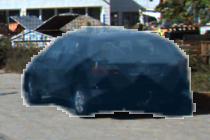}

  \caption{Top: category-level model before shape optimization. Bottom: instance-level model after shape optimization.}
  \label{fig:car_KITTI_3D}
\end{figure}

\begin{figure}[t]
  \centering
  \includegraphics[width=0.47\linewidth,valign=h,trim=200mm 0mm 0mm 30mm, clip]{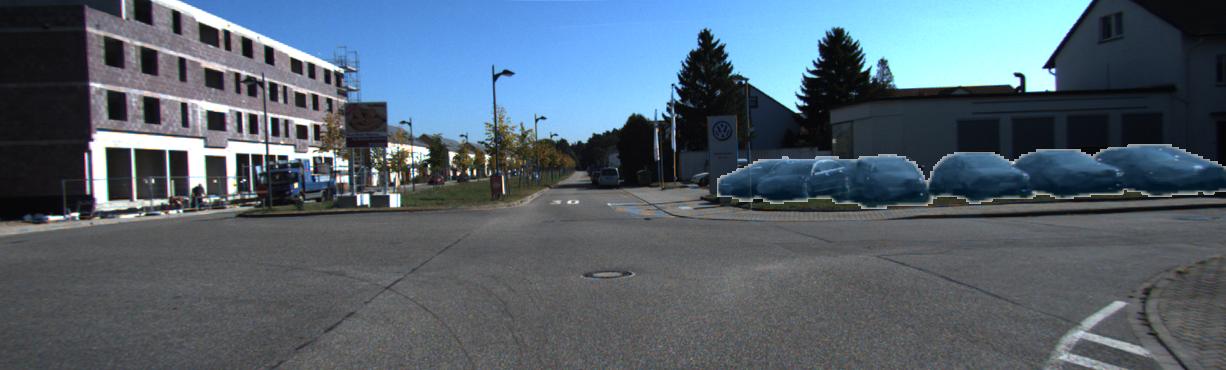}
  \includegraphics[width=0.47\linewidth,valign=h]{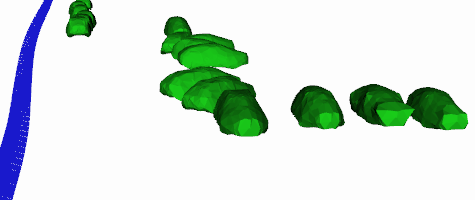}
  \includegraphics[width=0.47\linewidth,valign=h,trim=200mm 0mm 0mm 30mm, clip]{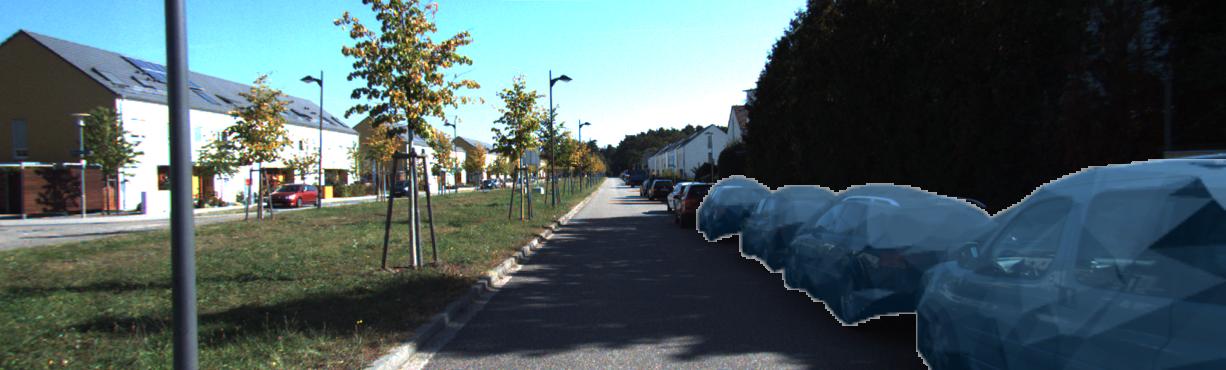}
  \includegraphics[width=0.47\linewidth,valign=h]{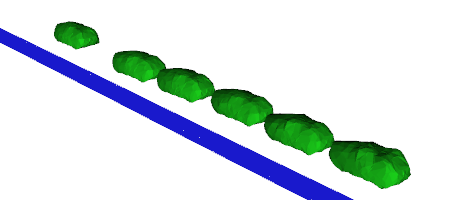}
  \caption{Left: 2D observation of mesh models. Right: corresponding 3D configuration. Trajectory in blue.}
  \label{fig:map_KITTI_3D}
\end{figure}

We evaluated the ability of the proposed localization and mapping technique to optimize the camera trajectory and reconstruct object poses and shapes using both simulated and real data. Our experiments used images from a monocular camera and inertial odometry information and focused on detecting, localizing and reconstructing cars object because 1) KITTI dataset with multiple cars is widely used and 2) there is pre-trained object detector and keypoint extractor for cars. We represented cars using a symmetric mesh model with 642 vertices and 1280 faces.

\subsection{Simulation Dataset}
To model the real mechanism of IMU, we chose a sub-sequence IMU measurement and associated groundtruth pose from synchronized KITTI odometry dataset. We collected camera images following the groundtruth pose in a simulated Gazebo environment populated with car mesh models, so that we simulated a real camera-IMU sensor (see Fig.~\ref{fig:loc_sim}). The car models were annotated with keypoints and both the car surface and the keypoints were colored in contrasting colors to simplify the semantic segmentation and keypoint detection tasks. The simulated experiments used ground-truth data association among the observations. We evaluated both the localization and the mapping tasks.

For the localization task, we used a sequence with 70 frames and synchronized IMU measurements and 6 known cars were placed around. We initialized the estimation by predicting the transformation between two camera poses based on the IMU odometry. Then, we optimized the predicted camera pose by solving problem~\eqref{eq:problem} and used the IMU to predict the next pose. An example camera trajectory and the associated localization results are shown in Fig.~\ref{fig:loc_sim}. We can see that our optimization successfully reduced the error accumulated from IMU integration.

The mapping performance was evaluated on a sequence of images obtained from different views of a single object (see Fig.~\ref{fig:map_sim_viewpoints}). The optimization was initialized using a generic category-level car mesh and its vertices were optimized based on the detected keypoints and segmentation masks. The mapping quality is evaluated qualitatively using the Intersection over Union (IoU) ratio between the predicted and groundtruth car masks volumes. In detail, the mask IoU compares the area differences between predicted binary car masks, while the voxel IoU compares the voxelized volume of the predicted and groundtruth car models. Fig.~\ref{fig:map_sim_stats} shows the dependence of the mapping accuracy on the number of different views used. The optimized car meshes are shown in Fig.~\ref{fig:map_sim}. The differences among car models are clearly visible in the reconstructed meshes and their shapes are very close to the corresponding groundtruth shapes. Using only a few views, the optimization process is able to deform the mesh vertices to fit the segmentation masks but not necessarily align the estimated model with the real 3-D shape. As more observations become available, the 3D IoU increases, which makes sense since different views can provide information about additional instance-level characteristics. Based only on 3 views, the IoU reaches over 0.8, while the generic category-level mesh has an average IoU of 0.63 with respect to the different object instances.

\begin{table}[t]
  \caption{2D projection mIoU with respect to object segmentation on three KITTI sequence~\cite{geiger2013vision}}
  \label{tab:2DIoU}
  \centering
  \begin{tabular}{c|c|c|c}
		\multirow{2}{*}{Dataset} & 09\_26 & 09\_26 & 09\_30 \\
		 & 0048 & 0035 & 0020 \\
		\hline
		Frames & 22 & 131 & 1101 \\
		Detected objects & 6 & 28 & 77 \\
		\hline
		Single image mesh prediction~\cite{kanazawa2018learning} & 0.692 & 0.642 & 0.641\\
		With pose estimation & 0.735 & 0.656 & 0.689 \\
		With pose and shape estimation & \textbf{0.778} & \textbf{0.675} & \textbf{0.725}\\
  \end{tabular}
\end{table}

\subsection{KITTI Dataset}
Experiments with real observations were carried out using the KITTI dataset~\cite{geiger2013vision}. We choose three sequences with different lengths. The experiments used the ground-truth camera poses and evaluated only the mapping task. The object detector, semantic segmentation~\cite{massa2018mrcnn} and the keypoint detector~\cite{zhou2018starmap} algorithms used pre-trained weights. 
Semantic observations were collected as described in Sec.~\ref{sec:approach_perception}. The poses and shapes of the detected cars were initialized and optimized as described in Sec.~\ref{sec:approach_details}. Fig.~\ref{fig:loc_KITTI} shows a bird-eye view of the estimated car poses and compares the results with the ground truth car positions provided in~\cite{atanasov2016localization}.
The poses and shapes of 56 out of 62 marked cars were reconstructed, with an average position error across all cars of 1.9 meters. 
Fig.~\ref{fig:map_KITTI} shows some estimated 3-D car mesh models projected back onto the camera images. Fig.~\ref{fig:car_KITTI_3D} compares the differences between the category-level and instance-specific models. Fig.~\ref{fig:map_KITTI_3D} shows the estimated model shapes and poses in 3D. 
Since groundtruth object shapes are not available, we evaluate the quality of shape reconstruction based on the 2D IoU compared with the observed instance segmentation masks. We trained a single-image mesh predictor~\cite{kanazawa2018learning} on car data from the PASCAL3D+ dataset~\cite{xiang_wacv14} and calculated its mean IoU for individual objects over multiple frames. Table \ref{tab:2DIoU} shows that our multi-view optimization method improves the IoU by leveraging semantic information from multiple images. The reconstruction quality on the real dataset is limited by the accuracy of the semantic information because the optimization objective is to align the predicted car shapes with the observed semantic masks and keypoints. The viewpoint changes on the real dataset are smaller, making the reconstruction task harder than in simulation. The pose estimation relies heavily on the keypoint detections, which in some cases are not robust enough. Nevertheless, our approach is able to generate accurate instance-specific mesh models in an environment containing occlusion and different lighting conditions.

\section{Conclusion}
\label{sec:conclusion}

This work demonstrates that object categories, shapes and poses can be recovered from visual semantic observations. The key innovation is the development of differentiable keypoint and segmentation mask projection models that allow object shape to be used for simultaneous semantic mapping and camera pose optimization. In contrast with existing techniques, our method generates accurate instance-level reconstructions of multiple objects, incorporating multi-view semantic information. Future work will extend the mesh reconstruction to multiple object categories and will focus on data association techniques for object re-identification and loop closure. Our ultimate goal is to develop an online SLAM algorithm that unifies semantic, geometric, and inertial measurements to allow rich environment understanding and contextual reasoning.

% including non-rigid objects

%In this work, we have demonstrated that a deformable mesh model can be used to recover the pose and instance-specific shape of objects using semantic observations. Compared with previous method using category-level model, our mapping method can generate instance-level mesh reconstruction to represent more details of the objects. The observation model forms the constraints between camera poses and object states such that it can also help with localization task. 

%In the future work, we will combine this with geometric feature points observations and IMU measurement to implement a complete SLAM system. Also we want to extend the mesh reconstruction to multiple categories of objects and try the application of re-identification of objects for loop closure detection.

%\appendices
%\input{tex/Appendix.tex}

%==================================================================%
% References
{\small
\bibliographystyle{cls/IEEEtran}
\bibliography{bib/ref.bib}
}
\end{document}